\documentclass[letterpaper, 10 pt, conference]{ieeeconf}   
\IEEEoverridecommandlockouts                    
\let\labelindent\relax
\usepackage{enumitem}

\usepackage{mathtools,amsthm}

\usepackage{graphics} 
\usepackage{epsfig}
\usepackage{times} 
\usepackage{amsmath,amssymb,amsfonts}
\usepackage{mathrsfs}
\usepackage{algorithm}
\usepackage{algorithmic}
\usepackage{xurl}
\usepackage{enumitem}
\usepackage{hyperref}
\usepackage{graphicx}
\usepackage{textcomp}
\usepackage{varioref}
\usepackage{import}
\usepackage{framed,xcolor}
\usepackage{color}
\usepackage{comment}
\usepackage{multirow}
\usepackage{epstopdf}   
\usepackage{psfrag}
\usepackage{breqn}
\usepackage{float}
\usepackage{mathtools}
\usepackage{booktabs}
\usepackage{subcaption}
\usepackage{soul}
\usepackage{amsbsy}
\usepackage[bottom]{footmisc}
\usepackage{lineno}
\usepackage{cleveref}
\usepackage{microtype}
\usepackage{siunitx}
\usepackage{comment}
\usepackage{multicol}
\newtheorem{theorem}{Theorem}

\newtheorem{Assumption}{Assumption}
\definecolor{arash}{rgb}{0.8,0.8,1}
\definecolor{seb}{rgb}{0.8,1,0.8}
\definecolor{seb2}{rgb}{0.5,.5,1}
\definecolor{arash2}{rgb}{0,.5,0}
\definecolor{wenqi}{rgb}{1,.75,0.79}
\definecolor{wenqi2}{rgb}{1,.75,0.79}
\graphicspath{{Figures/}}

\graphicspath{{Figures/}}
\newcommand{\vect}[1]{\ensuremath{\boldsymbol{\mathrm{#1}}}}
\usepackage{graphicx}
\usepackage{tabularx,booktabs}
\usepackage{color}
\usepackage{multicol}
\usepackage{amsmath,amssymb}
\usepackage{amsfonts}
\usepackage{textcomp}
\usepackage{psfrag}
\usepackage{multimedia}
\usepackage{fancybox}
\usepackage{comment}
\usepackage{soul}
\makeatletter
\newcommand{\biggg}{\bBigg@{1.6}}  

\makeatother

\definecolor{seb}{rgb}{0.8,1,0.8}
\definecolor{arash}{rgb}{0.8,0.8,1}

\newcounter{lastnote}

\usepackage{geometry}
\title{\LARGE \bf
Cost-Matching Model Predictive Control for Efficient\\Reinforcement Learning in Humanoid Locomotion}
\author{Wenqi Cai$^{1}$,  
Kyriakos G. Vamvoudakis$^{2}$, 
Sébastien Gros$^{3}$,
and Anthony Tzes$^{1}$
\thanks{$^{1}$Wenqi Cai and Anthony Tzes are with the Electrical Engineering Program, New York University Abu Dhabi (NYUAD), 129188, UAE. Email: wenqi.cai@nyu.edu; anthony.tzes@nyu.edu}
\thanks{$^{2}$Kyriakos G. Vamvoudakis is with the School of Aerospace Engineering, Georgia Institute of Technology, Atlanta, GA 30332, United States. Email: kyriakos@gatech.edu}
\thanks{$^{3}$Sébastien Gros is with the Department of Engineering Cybernetics, Norwegian University of Science and Technology (NTNU), Trondheim, Norway. Email: sebastien.gros@ntnu.no}
\thanks{This work was supported in part by: a) NSF under grant Nos. SLES-$2415479$, CPS-$2227185$, b) NASA ULI under grant No. $80$NSSC$25$M$7104$, and c) the NYUAD Center for AI and Robotics, funded by Tamkeen under the NYUAD Research Institute Award CG010.}
}
\begin{document}
\maketitle
\thispagestyle{empty}
\pagestyle{empty}
%%%%%%%%%%%%%%%%%%%%%%%%%%%%%%%%%%%%%%%%%%%%%%%%%%%%%%%%%%%%%%%%%%%%%%%%%%%%%%%%
\begin{abstract}
In this paper, we propose a cost-matching approach for optimal humanoid locomotion within a Model Predictive Control (MPC)-based Reinforcement Learning (RL) framework. A parameterized MPC formulation with centroidal dynamics is trained to approximate the action-value function obtained from high-fidelity closed-loop data. Specifically, the MPC cost-to-go is evaluated along recorded state–action trajectories, and the parameters are updated to minimize the discrepancy between MPC-predicted values and measured returns. This formulation enables efficient gradient-based learning while avoiding the computational burden of repeatedly solving the MPC problem during training. The proposed method is validated in simulation using a commercial humanoid platform. Results demonstrate improved locomotion performance and robustness to model mismatch and external disturbances compared with manually tuned baselines.
\end{abstract}
% %%%%%%%%%%%%%%%%%%%%%%%%%%%%%%%%%%%%%%%%%%%%%%%%%%%%%%%%%%%%%%%%%%%%%%%%%%%%%%%%
\section{Introduction}
\label{sec:introduction}
Humanoid robotics has advanced rapidly in recent years. Humanoid locomotion poses particular challenges due to intermittent contacts, high degrees of freedom, and strict safety constraints, while also requiring robustness to disturbances and modeling errors~\cite{li2026comprehensive}.

Model Predictive Control (MPC) has emerged as a dominant paradigm for humanoid locomotion because it provides a systematic mechanism to encode task objectives and enforce constraints through online optimization. To satisfy real-time requirements, many MPC pipelines rely on simplified or reduced-order models that enable lightweight computation and high update rates~\cite{chen2024beyond}. However, such simplified models may neglect important effects present during high-fidelity execution (e.g., contact realization, inertial coupling, and actuator or tracking-stack abstractions). Consequently, robust performance often relies heavily on expert tuning of cost weights and terminal penalties~\cite{katayama2023model}. Alternatively, richer formulations that tighten kinematic consistency or incorporate more detailed dynamics~\cite{galliker2022planar} can improve fidelity but typically increase computational complexity and complicate real-time deployment~\cite{gu2025humanoid}. In practice, this trade-off is often addressed using predictive--reactive hierarchies, where an MPC planner generates reference trajectories that are executed by a higher-rate tracking module~\cite{zhang2026gravity}. While effective for real-time operation, this architectural separation can amplify systematic planning--execution mismatches and make performance sensitive to cost design and tuning, particularly when operating with short horizons~\cite{carpentier2021recent}.

Reinforcement Learning (RL) has demonstrated strong capability in learning robust locomotion behaviors directly from data with efficient runtime execution~\cite{peters2003reinforcement}. However, humanoid locomotion represents a challenging learning regime. Training from scratch is often sample-inefficient and sensitive to reward design and curriculum choices, and successful deployments typically rely on large-scale simulation pipelines with additional mechanisms to mitigate sim-to-real discrepancies (e.g., domain randomization~\cite{haarnoja2024learning} and system identification~\cite{li2024learning}). Furthermore, purely learned policies do not naturally enforce hard safety and contact constraints~\cite{zhang2022deep}. These limitations have motivated control--learning formulations that connect RL with adaptive optimal control concepts and provide explicit convergence or stability analyses under disturbances~\cite{chen2025robot, gao2026output}. In this direction, MPC--learning hybrids aim to retain constrained optimization for deployment while leveraging learning to improve performance under model mismatch and uncertainty~\cite{zhang2022model, vamvoudakis2021handbook}.

In this context, this work builds upon our prior MPC-based Reinforcement Learning (MPC-RL) framework—initiated in~\cite{gros2019data} and further developed in subsequent works (e.g.,~\cite{zanon2020safe,cai2023energy,cai2021mpc})—which treats a parameterized MPC as a structured value/action-value function approximator or policy approximator. Within this framework, the predictive model, cost function, and constraints of the MPC are parameterized, and closed-loop behavior is improved by adjusting these parameters using RL signals. This approach allows learning to operate within the MPC structure, thereby preserving constraint handling and interpretability~\cite{cai2023learning}. However, a key limitation arises in complex, time-critical humanoid locomotion stacks: standard gradient-based MPC-RL methods typically require repeatedly solving the MPC optimization within the learning loop, making training prohibitively expensive when the MPC itself is already operating near real-time computational limits.

Motivated by this challenge, we propose Cost-Matching MPC (CM-MPC), an efficient learning procedure that preserves the MPC-RL perspective while avoiding solve-in-the-loop training. The central idea is to learn the MPC parameters $\vect\theta$ by minimizing the discrepancy between an MPC surrogate cost-to-go $Q^{\mathrm{MPC}}_{\vect\theta}$ and a measured long-horizon return $Q^{\mathrm{meas}}$ computed from closed-loop trajectories. Rather than differentiating through repeated MPC solves, $Q^{\mathrm{MPC}}_{\vect\theta}$ is evaluated by rolling out the parameterized predictive model along recorded action segments and accumulating parameterized stage costs, together with differentiable penalties for state-constraint violations. This formulation enables efficient gradient-based learning while maintaining deployment as a standard constrained receding-horizon MPC controller.

\textbf{Contributions.} The contributions of this work are twofold. First, we introduce a cost-matching learning framework that extends our prior MPC-RL formulation while substantially reducing the computational burden during training by eliminating the need for repeated MPC solves in the learning loop. The proposed framework is generic and can be applied to systems formulated as Optimal Control Problems (OCPs). Second, we demonstrate the approach on humanoid locomotion by developing a constrained OCP formulation based on centroidal dynamics and validating the learning procedure in high-fidelity simulation. An open-source implementation of the framework is also provided at \url{https://github.com/RISC-NYUAD/humanoid_mpc_cost_matching}.
%%%%%%%%%%%%%%%%%%%%%%%%%%%%%%%%%%%%%%%%%%%%%%%%%%%%%%%%%%%%%%%%%%%%%%%%%%%%%%%%%%%%%%%%%%%%%%%%%
%%%%%%%%%%%%%%%%%%%%%%%%%%%%%%%%%%%%%%%%%%%%%%%%%%%%%%%%%%%%%%%%%%%%%%%%%%%%%%%%%%%%%%%%%%%%%%%%%
\section{Problem Formulation}
\label{sec:problem_formulation}
Consider humanoid locomotion control within a hierarchical control architecture \cite{sleiman2021unified}. In this architecture, an upper-level nonlinear MPC planner operates at a lower frequency and generates reference commands, while a high-frequency Proportional Derivative (PD)-type tracking controller maps the MPC commands to joint torques (from $\vect u$ to $\tilde{\vect u}$), as shown in Fig.~\ref{fig:MPC-RL_framework}.
\begin{figure}[htbp]
  \centering
  \includegraphics[width=0.91\linewidth]{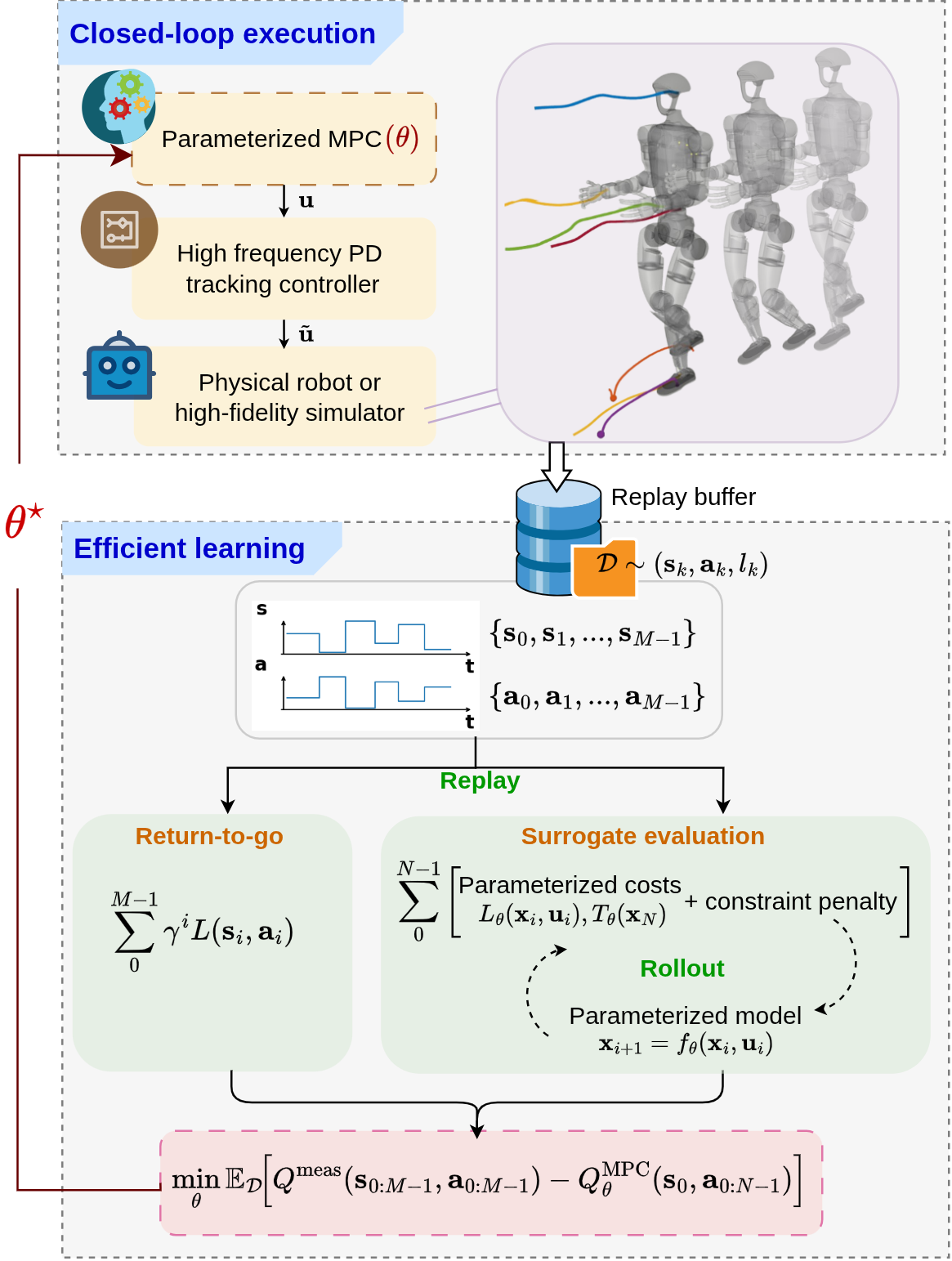}
  \caption{Cost-Matching MPC-RL framework for humanoids.}
  \label{fig:MPC-RL_framework}
\end{figure}
%%%%%%%%%%%%%%

In this work, we shall focus on the learning-based parameterization of the MPC layer, while the low-level tracking controller is treated as a fixed module. Fig.~\ref{fig:MPC-RL_framework} provides an overview of the proposed CM-MPC framework for humanoids, and the remaining components are introduced in the subsequent sections.

%%%%%%%%%%%%%%%%%%%%%%%%%%%%%%%%%%%%%%%%%%%%%%%%%%%%%%%%%%%%%%%%%%
\subsection{Centroidal Dynamics}
\label{subsec:system_dynamics}

A complete \textit{centroidal dynamics} model is adopted in this work, building on the OCS2 centroidal formulation~\cite{OCS2,sleiman2021unified}. Let $\mathcal{I}$ denote the inertial (world) frame and $\mathcal{B}$ the base frame attached to the robot trunk. The generalized coordinates of the robot are defined as $\vect q = [\vect p_b^\top, \vect \theta_b^\top, \vect q_j^\top]^\top \in \mathbb{R}^{6+n_j}$, where $\vect p_b \in \mathbb{R}^3$ denotes the base position in $\mathcal{I}$, $\vect \theta_b \in \mathbb{R}^3$ represents the base orientation (ZYX Euler angles), and $\vect q_j \in \mathbb{R}^{n_j}$ denotes the joint angles.

The state $\vect x \in \mathbb{R}^{12+n_j} :=  \begin{bmatrix} \vect h_{\mathrm{lin}}^\top & \vect h_{\mathrm{ang}}^\top & \vect q^\top \end{bmatrix}^\top$ concatenates the centroidal momentum $\vect h$ and the configuration $\vect q$, where $\vect h_{\mathrm{lin}}, \vect h_{\mathrm{ang}} \in \mathbb{R}^3$ denote the linear and angular momenta expressed in $\mathcal{I}$. The control input $\vect u \in \mathbb{R}^{12+n_j}$ consists of the contact wrenches for the left ($L$) and right ($R$) feet together with the joint velocities
$
\vect u := \begin{bmatrix} \vect w_{c,L}^\top & \vect w_{c,R}^\top & \vect v_j^\top \end{bmatrix}^\top,
$
where $\vect w_{c,i} = [\vect f_{c,i}^\top, \vect m_{c,i}^\top]^\top \in \mathbb{R}^6$ represents the contact force and moment exerted at the $i$-th foot ($i \in \{L, R\}$), and $\vect v_j \in \mathbb{R}^{n_j}$ denotes the commanded joint velocity.

The system dynamics $\dot{\vect x} = \vect f_{\mathrm{dyn}}(\vect x, \vect u)$ is
\begin{subequations}
\label{eq:centroidal_dynamics}
\begin{align}
\dot{\vect h}_{\mathrm{lin}} &= \sum_{i \in \{L, R\}} \vect f_{c,i} + M \vect g, \\
\dot{\vect h}_{\mathrm{ang}} &= \sum_{i\in\{L,R\}} \vect r_{c,i}(\vect q)\times \vect f_{c,i} + \vect m_{c,i},\\
\dot{\vect q} &= \vect \Xi(\vect q) \begin{bmatrix} \vect v_b^\top & \vect \omega_b^\top & \vect v_j^\top \end{bmatrix}^\top,
\end{align}
\end{subequations}
where $M$ denotes the total mass and $\vect g$ is the gravitational acceleration. The vector $\vect r_{c,i}$ denotes the moment arm from the Center of Mass (CoM) to the contact point of foot $i$, expressed in the inertial frame $\mathcal{I}$. The matrix $\vect \Xi(\vect q)$ maps generalized velocities to the configuration derivative $\dot{\vect q}$. These transformation-related quantities can be efficiently computed using rigid-body algorithms provided by the Pinocchio library~\cite{carpentier2019pinocchio}. The base velocity variables $(\vect v_b, \vect \omega_b)$ are determined by the centroidal momentum $\vect h$ and the joint velocities $\vect v_j$ through
$\vect h = \vect A(\vect q) \begin{bmatrix} \vect v_b \\ \vect \omega_b \\ \vect v_j \end{bmatrix} = \vect A_b(\vect q) \begin{bmatrix} \vect v_b \\ \vect \omega_b \end{bmatrix} + \vect A_j(\vect q) \vect v_j$, where $\vect A(\vect q)$ is the Centroidal Momentum Matrix (CMM)~\cite{orin2013centroidal}, which captures the inertial coupling between base and limb motions. The matrices $\vect A_b$ and $\vect A_j$ denote the corresponding block partitions of the CMM. The base velocity can therefore be obtained as
$\begin{bmatrix} \vect v_b \\ \vect \omega_b \end{bmatrix} = \vect A_b(\vect q)^{-1} \left( \vect h - \vect A_j(\vect q) \vect v_j \right)$, ensuring that the kinematic evolution remains consistent with the centroidal momentum dynamics.

%%%%%%%%%%%%%%%%%%%%%%%%%%%%%%%%%%%%%%%%%%%%%%%%%%%%%%%%%%%%%%%%%%%%%%%%%%%%%%%%
\subsection{Constraints}
\label{subsec:constraints}

The feasible state–input set is defined through equality and inequality constraints. To distinguish phase-dependent constraints, let $\mathcal P_{\mathrm{st}}$ and $\mathcal P_{\mathrm{sw}}$ denote the sets of stance and swing phases, respectively.

\subsubsection{Inequality Constraints}

\begin{itemize}[leftmargin=*, nosep]

\item Joint Limits. The joint configurations are restricted by mechanical limits: $\vect q_{j}^{\min} \le \vect q_j \le \vect q_{j}^{\max}.$

\item Foot Collision. To prevent self-collision, the minimum distance between the left and right feet, computed using Signed Distance Fields (SDF) $\Phi_{\text{sdf}}$, must maintain a safety margin $d_{\mathrm{safe}}$: $\Phi_{\text{sdf}}(\vect q) \ge d_{\mathrm{safe}}.$

\item Friction Cone. To avoid foot slippage, the contact force for any active foot must lie within the friction cone: $\sqrt{\left(f_{c,i}^x\right)^2+\left(f_{c,i}^y\right)^2}\le \mu f_{c,i}^z, \,\forall i \in \mathcal{P}_{st}.$

\item Center of Pressure (CoP). To maintain contact stability with foot dimensions $(d_x,d_y)$, the CoP must remain within the support region: $| m_{c,i}^x | \le d_y f_{c,i}^z, \, | m_{c,i}^y | \le d_x f_{c,i}^z, \, \forall i \in \mathcal{P}_{st}.$

\end{itemize}

%%%%%%%%%%%%%%%%%%%%%%%%%%%%%%%%%%%%%%%
\subsubsection{Equality Constraints}

\begin{itemize}[leftmargin=*, nosep]

\item Zero Velocity. During stance, the foot maintains rigid ground contact and the linear velocity of the foot contact frame is constrained to zero: $\vect v_{c,i}=\vect J_{c,i}(\vect q)\dot{\vect q}=\vect 0,\,\forall i\in\mathcal P_{\mathrm{st}},$
where $\vect J_{c,i}(\vect q)$ denotes the contact Jacobian.

\item Zero Wrench. Feet in the swing phase are contact-free, and their contact wrench must vanish: $\vect w_{c,i} = \vect 0, \, \forall i \in \mathcal{P}_{\mathrm{sw}}.$

\item Normal Velocity Tracking. The normal ($z$-axis) velocity of the swing foot is constrained to follow a reference profile $v_{\mathrm{ref}}^z$ to ensure accurate landing: $v_{c,i}^z = v_{\mathrm{ref}}^z(t), \, \forall i \in \mathcal{P}_{\mathrm{sw}}.$

\end{itemize}
%%%%%%%%%%%%%%%%%%%%%%%%%%%%%%%%%%%%%%%%%%%%%%%%%%%%%%%%%%%%%%%%%%%%%%%%%%%%%%%%%%%%%%%%%%%%%%%%%
\subsection{Optimal Control Objective}
\label{subsec:mpc_cost}

Humanoid locomotion is formulated as a constrained OCP over a finite prediction horizon $N$. Given the centroidal dynamics and feasibility constraints, the planner seeks an optimal control sequence that minimizes
\begin{equation}
J_{\mathrm{MPC}} := T(\vect x_N) + \sum_{i=0}^{N-1} L(\vect x_i,\vect u_i),
\end{equation}
where $L(\cdot)$ and $T(\cdot)$ denote the stage and terminal costs, respectively. The stage cost is composed of
\begin{equation}
\label{eq:ocp_l}
L = L_{\mathrm{trac}} + L_{\mathrm{base}} + L_{\mathrm{com}} + L_{\mathrm{swin}} + L_{\mathrm{torq}},
\end{equation}
capturing reference tracking, dynamic stability, and physical plausibility. Specifically:
\begin{itemize}[leftmargin=*, nosep]

\item We enforce tracking of the reference state trajectory $\mathbf{x}_{\mathrm{ref}}$ while regularizing the control effort:
$L_{\mathrm{trac}} = \| \vect x - \vect x_{\mathrm{ref}} \|_{\vect Q}^2 + \| \vect u - \vect u_{\mathrm{ref}}\|_{\vect R}^2.$

\item A task-space kinematic penalty is imposed on a base-attached frame to regularize upper-body motion: $L_{\mathrm{base}}=\left\|\vect e_{\mathrm{base}}(\vect x,\vect u)\right\|_{\vect Q_{\mathrm{base}}}^{2}$, where $\vect e_{\mathrm{base}}\in\mathbb R^{18}$ collects the task-space tracking errors. This term mitigates upper-body oscillations and promotes an upright posture during locomotion.

\item To improve balance robustness, the horizontal CoM position is regularized toward the midpoint of the active foot contacts: $L_{\mathrm{com}}=\left\|
\vect p_{\mathrm{com}}^{xy}(\vect q)-
\vect p_{\mathrm{mid}}^{xy}(\vect q)
\right\|_{\vect Q_{\mathrm{com}}}^{2}$, where $\vect p_{\mathrm{com}}^{xy}$ denotes the CoM position projected onto the horizontal plane, and $\vect p_{\mathrm{mid}}^{xy}:=\tfrac{1}{2}\big(\vect p_{c,L}^{xy}(\vect q)+\vect p_{c,R}^{xy}(\vect q)\big)$ denotes the midpoint between the left and right foot contacts.

\item To regulate foot clearance and landing orientation, deviations of the swing-foot pose and twist from their references are penalized: $L_{\mathrm{swin}}=\sum_{i\in\mathcal P_{\mathrm{sw}}}\left\|\vect e_{\mathrm{sw},i}(\vect x,\vect u)\right\|_{\vect Q_{\mathrm{sw}}}^{2}$, where $\vect e_{\mathrm{sw},i}\in\mathbb R^{18}$ denotes the task-space tracking error of the swing-foot frame relative to its reference trajectory.

\item To discourage contact force distributions that generate large joint torques, the contact wrench is regularized through the induced generalized torques: $L_{\mathrm{torq}}=\sum_{i\in\mathcal P_{\mathrm{st}}}
\left\|\vect J_{c,i}(\vect q)^\top \vect w_{c,i}\right\|_{\vect Q_{\mathrm{torq}}}^{2}$, where $\vect J_{c,i}(\vect q)$ denotes the contact Jacobian.

\end{itemize}
The terminal cost is defined as a quadratic penalty on the terminal state,
\begin{equation}
T(\vect x_N)=\left\|\vect x_N-\vect x_{N}^{\mathrm{ref}}\right\|_{\vect Q_f}^{2}.
\end{equation}

The above OCP provides a structured mechanism for generating feasible locomotion behaviors. Nevertheless, model mismatch and unmodeled effects can lead to suboptimal MPC policies and degraded closed-loop performance. To address this issue, the next section introduces a data-driven adaptation of the MPC parameterization using rollout data collected during execution.

%
%%%%%%%%%%%%%%%%%%%%%%%%%%%%%%%%%%%%%%%%%%%%%%%%%%%%%%%%%%%%%%%%%%%%%%%%%%%%%%%%%%%%%%%%%%%%%%%%%
\section{Cost-Matching MPC-Based\\Reinforcement Learning}
\label{sec:cost_matching}

We present the cost-matching method using the general Markov decision process (MDP) notation in order to emphasize that the proposed approach applies to any problem that can be formulated as a constrained OCP and for which real-world data is available.

\subsection{Parameterized MPC as a Q Function}
\label{sec:cm_mpc_formulation}

Consider a discounted infinite-horizon MDP with state $\vect s\in\mathcal S$, action $\vect a\in\mathcal A$, stage cost $L(\vect s,\vect a)$, and discount factor $\gamma\in(0,1]$. The objective is to minimize the expected discounted cumulative cost. To obtain a structured and differentiable approximation of the optimal policy and its associated action-value function, we employ a finite-horizon parameterized model predictive control (MPC) scheme as a function approximator. The parameter vector $\vect\theta\in\mathbb R^{p}$ jointly parameterizes the predictive model and cost functions of the MPC, enabling data-driven adaptation in the presence of model mismatch and uncertainties.

\paragraph{Parameterized MPC scheme}

Given a measured state $\vect s$, the horizon-$N$ parameterized MPC problem is
\begin{subequations}\label{eq:cm_mpc}
\begin{align}
\min_{\vect x,\,\vect u}\;&
T_{\vect\theta}(\vect x_N)+\sum_{i=0}^{N-1}L_{\vect\theta}(\vect x_i,\vect u_i),
\label{eq:cm_mpc_cost}\\
\mathrm{s.t.}\;& \quad \forall \,\, i=0,\ldots,N-1 \nonumber\\
& \quad \vect x_{i+1}= \vect f_{\vect\theta}(\vect x_i,\vect u_i),
%& 
\quad \vect g(\vect u_i)\le \vect 0,
\label{eq:cm_mpc_ucon}\\
& \quad \vect h(\vect x_i,\vect u_i)\le \vect 0,\quad \vect c(\vect x_i,\vect u_i)= \vect 0,
\label{eq:cm_mpc_xcon_eqcon}\\
& \quad \vect h^{f}(\vect x_N)\le \vect 0,\quad \vect c^{f}(\vect x_N)= \vect 0,
\label{eq:cm_mpc_termcon_termeqcon}\\
& \quad \vect x_{0}=\vect s.
\label{eq:cm_mpc_init}
\end{align}
\end{subequations}
The predictive model $\vect f_{\vect\theta}$, stage cost $L_{\vect\theta}$, and terminal cost $T_{\vect\theta}$ are parameterized and adapted through learning.

\paragraph{MPC-induced action-value function}

To enable learning without repeatedly solving the optimization problem, we define the MPC-induced action-value along given actions recorded from the real system. For a length-$N$ action segment $\vect a_{0:N-1}:=\{\vect a_0,\dots,\vect a_{N-1}\}$, the predicted rollout is
\begin{equation}
\label{eq:cm_rollout_dyn}
\vect x_0=\vect s,~~
\vect x_{i+1}=\vect f_{\vect\theta}(\vect x_i,\vect a_i),
\quad i=0,\dots,N-1.
\end{equation}

Note that even when the recorded actions are feasible on the real system, the predicted trajectory generated by \eqref{eq:cm_rollout_dyn} may violate state-related constraints due to model mismatch. Consequently, constraint-violation penalties are incorporated directly into the rollout-based value evaluation. Define the per-stage and terminal violation residuals
\begin{equation}
\label{eq:cm_residual}
\vect r_i :=
\begin{bmatrix}
\big[\vect h_{\vect\theta}(\vect x_i,\vect a_i)\big]_+\\[2pt]
\vect c_{\vect\theta}(\vect x_i,\vect a_i)
\end{bmatrix},
\quad 
\vect r_N :=
\begin{bmatrix}
\big[\vect h_{\vect\theta}^{f}(\vect x_N)\big]_+\\[2pt]
\vect c_{\vect\theta}^{f}(\vect x_N)
\end{bmatrix},
\end{equation}
where $[\cdot]_+$ denotes the element-wise operator $[y]_+=\max(y,0)$. Let the weights $\vect W,\vect W_f\succ \vect 0$, and define
\begin{equation}
\label{eq:cm_penalty}
\phi(\vect r_i):=\|\vect r_i\|_{\vect W}^{2},\quad
\phi_f(\vect r_N):=\|\vect r_N\|_{\vect W_f}^{2}.
\end{equation}
The MPC-induced action-value function is then defined as
\begin{align}
Q_{\vect\theta}^{\mathrm{MPC}}(\vect s,\vect a_{0:N-1})
:=&T_{\vect\theta}(\vect x_N)
+\phi_f(\vect r_N) +\nonumber\\
&\sum_{i=0}^{N-1}\Big(
L_{\vect\theta}(\vect x_i,\vect a_i)
+\phi(\vect r_i)
\Big),
\label{eq:cm_Q_mpc_def}
\end{align}
where $\{\vect x_i\}$ and $\{\vect r_i\}$ are generated by \eqref{eq:cm_rollout_dyn} and \eqref{eq:cm_residual}. Importantly, \eqref{eq:cm_Q_mpc_def} can be evaluated through forward rollout and cost accumulation, and is differentiable with respect to $\vect\theta$, without requiring the solution of \eqref{eq:cm_mpc} during training.

%%%%%%%%%%%%%%%%%%%%%%%%%%%%%%%%%%%%%%%%%%%%%%%%%%%%%%%%%%%%%%%%%%%%%%%%%%%%%%%%%%%%%%%%%%%%%%%%%
\subsection{Action-Value Evaluation via Data Rollouts}
\label{sec:cm_data_value}

Consider a dataset of trajectories collected from the real environment $\mathcal D = \{(\vect s_k,\vect a_k,\ell_k)\}_{k=0}^{M-1}$, where $(M\!\gg\!N)$ denotes the trajectory length and $\ell_k = L(\vect s_k,\vect a_k)$ is the observed stage cost. The trajectories need not be optimal. For each time index $k$, we define the measured discounted return-to-go along the recorded trajectory as
\begin{equation}
Q^{\mathrm{meas}}(\vect s_k,\vect a_k) :=
\sum_{i=0}^{M-k-1}\gamma^i\,L(\vect s_{k+i},\vect a_{k+i}).
\label{eq:cm_Q_meas}
\end{equation}
This quantity represents the long-term cost accumulated by the system under the executed actions. In the proposed framework, $Q^{\mathrm{meas}}$ serves as the learning target, while $Q_{\vect\theta}^{\mathrm{MPC}}$ provides a finite-horizon, model-based approximation evaluated along the same recorded action segment. The discrepancy between these quantities primarily reflects a systematic mismatch between the real system and the parameterized predictive model.

Given an index $k$, we compute $Q^{\mathrm{MPC}}_{\vect\theta}$ by initializing \eqref{eq:cm_rollout_dyn} at $\vect x_0=\vect s_k$ and rolling out the parameterized model over the recorded action segment $(\vect a_k,\dots,\vect a_{k+N-1})$. This estimate is compared with $Q^{\mathrm{meas}}$, which accumulates the discounted costs along the measured trajectory from $k$ to $M\!-\!1$.

From an RL perspective, $Q^{\mathrm{MPC}}_{\vect\theta}$ and $Q^{\mathrm{meas}}$ represent two evaluations of the action-value at the same anchor state $\vect s_k$ for the executed action segment: a finite-horizon model-based surrogate and the measured rollout return. While typically $M\!\gg\!N$, the discount factor $\gamma$ in $Q^{\mathrm{meas}}$ and the terminal shaping $T_{\vect\theta}(\vect x_N)$ in $Q^{\mathrm{MPC}}$ compensate for the horizon mismatch. Discounting attenuates far-future costs, while the terminal term summarizes the residual cost-to-go beyond the prediction horizon. This observation highlights the importance of terminal cost design in short-horizon MPC and motivates learning its relative weight to align the surrogate with long-term performance.

%%%%%%%%%%%%%%%%%%%%%%%%%%%%%%%%%%%%%%%%%%%%%%%%%%%%%%%%%%%%%%%%%%%%%%%%%%%%%%%%%%%%%%%%%%%%%%%%%
\subsection{Cost-Matching Objective and Gradient}
\label{sec:cm_cost_matching}

The key idea behind the learning objective is to adjust $\vect\theta$ such that the MPC-induced value $Q^{\mathrm{MPC}}_{\vect\theta}$ evaluated along recorded action segments matches the long-term return $Q^{\mathrm{meas}}$ observed on the real system.

Given the dataset $\mathcal D$ and the measured returns \eqref{eq:cm_Q_meas}, the cost-matching objective is
\begin{equation}
\label{eq:cm_loss}
\min_{\vect\theta}\;
\mathcal L(\vect\theta)
:=
\mathbb E_{\mathcal D}\!\left[
\big(
Q^{\mathrm{MPC}}_{\vect\theta}
-
Q^{\mathrm{meas}}
\big)^2
\right].
\end{equation}
The expectation in \eqref{eq:cm_loss} is approximated using empirical averages over mini-batches. For a fixed recorded action segment, $Q^{\mathrm{MPC}}_{\vect\theta}$ is obtained via the forward rollout \eqref{eq:cm_rollout_dyn} and accumulation of the cost and penalty terms in \eqref{eq:cm_Q_mpc_def}. Therefore, $\nabla_{\vect\theta}Q^{\mathrm{MPC}}_{\vect\theta}$ can be computed by differentiating through the rollout recursion and the corresponding cost and penalty evaluations. The loss gradient is

\begin{equation}
\label{eq:cm_grad}
\nabla_{\vect\theta}\mathcal L(\vect\theta)
=
\mathbb E_{\mathcal D}\!\left[
2\big(
Q^{\mathrm{MPC}}_{\vect\theta}
-
Q^{\mathrm{meas}}
\big)
\nabla_{\vect\theta} Q^{\mathrm{MPC}}_{\vect\theta}
\right].
\end{equation}
Since training does not require solving \eqref{eq:cm_mpc}, gradient evaluation reduces to differentiating a length-$N$ rollout and scales linearly with the horizon length $N$. The parameters are updated using a first-order optimization method,
\begin{equation}
\label{eq:cm_update}
\vect\theta \leftarrow \vect\theta - \alpha \nabla_{\vect\theta}\mathcal L(\vect\theta),
\end{equation}
with step size $\alpha>0$. After convergence, the learned parameters $\vect\theta^\star$ define a structured MPC controller deployed online by solving \eqref{eq:cm_mpc} in a receding-horizon fashion.
%%%%%%%%%%%%%%%%%%%%%%%%%%%%%%%%%%%%%%%%%%%%%%%%%%%%%%%%%%%%%%%%%%%%%%%%%%%%%%%%%%%%%%%%%%%%%%%%%%%%%%%%%%%%%%%%
\subsection{Convergence of Cost Matching}
\label{subsec:cm_convergence}
The cost-matching update \eqref{eq:cm_update} is a stochastic optimization of \eqref{eq:cm_loss}. The following theorem establishes convergence to a first-order stationary point under standard smoothness and bounded-variance assumptions.
\begin{Assumption}
\label{assumption}
For each sample $\xi=(\vect s,\vect a_{0:N-1},Q^{\mathrm{meas}})$, define $q_{\xi}(\vect\theta):=Q_{\vect\theta}^{\mathrm{MPC}}(\vect s,\vect a_{0:N-1})$. Assume that $q_{\xi}(\vect\theta)$ is continuously differentiable in $\vect\theta$, and that there exist constants $B_Q,G_Q,L_Q>0$ such that, for all samples $\xi$ and all $\vect\theta,\vect\theta'$,
\[
|q_{\xi}(\vect\theta)-Q^{\mathrm{meas}}|\le B_Q,\qquad
\|\nabla q_{\xi}(\vect\theta)\|\le G_Q,
\]
\[
\|\nabla q_{\xi}(\vect\theta)-\nabla q_{\xi}(\vect\theta')\|
\le L_Q\|\vect\theta-\vect\theta'\|.
\]
Moreover, the mini-batch gradient estimator $\widehat{\nabla}\mathcal L(\vect\theta)$ used in \eqref{eq:cm_update} is unbiased and has bounded variance, i.e.,
\[
\mathbb E[\widehat{\nabla}\mathcal L(\vect\theta)\mid \vect\theta]
=
\nabla\mathcal L(\vect\theta),\,\, \mathbb E\!\left[\left\|\widehat{\nabla}\mathcal L(\vect\theta)-\nabla\mathcal L(\vect\theta)\right\|^2\middle|\vect\theta\right]
\le \sigma^2.
\]\hfill  $\square$
\end{Assumption}

\begin{theorem}[Convergence of cost matching]
\label{thm:cm_convergence}
Suppose that Assumption 1 holds and the loss $\mathcal L$ in \eqref{eq:cm_loss} has Lipschitz-continuous gradient with constant $L_{\mathcal L}=2(G_Q^2+B_QL_Q)$. Let the parameter sequence $\{\vect\theta_j\}$ be generated by $\vect\theta_{j+1}=\vect\theta_j-\alpha \widehat{\nabla}\mathcal L(\vect\theta_j)$,
with a constant step size $0<\alpha\le 1/L_{\mathcal L}$. Then, for any $K\ge 1$,
\[
\frac{1}{K}\sum_{j=0}^{K-1}\mathbb E\!\left[\|\nabla\mathcal L(\vect\theta_j)\|^2\right]
\le
\frac{2(\mathcal L(\vect\theta_0)-\mathcal L_{\inf})}{\alpha K}
+\alpha L_{\mathcal L}\sigma^2,
\]
where $\mathcal L_{\inf}:=\inf_{\vect\theta}\mathcal L(\vect\theta)$. In the full-batch case ($\sigma=0$), $\mathcal L(\vect\theta_j)$ is monotonically nonincreasing and $\nabla \mathcal L(\vect\theta_j)\to \vect 0$ as $j\to\infty$.
\end{theorem}

\begin{proof}
For a single sample $\xi$, the gradient of the squared matching error is
\[
\nabla \big(q_{\xi}(\vect\theta)-Q^{\mathrm{meas}}\big)^2
=
2\big(q_{\xi}(\vect\theta)-Q^{\mathrm{meas}}\big)\nabla q_{\xi}(\vect\theta).
\]
By Assumption~1, this gradient is Lipschitz with constant $2(G_Q^2+B_QL_Q)$, hence $\nabla \mathcal L$ is also Lipschitz with the same constant. Using the standard descent lemma \cite{nesterov2013introductory} for smooth objectives, we have
\[
\mathcal L(\vect\theta_{j+1})
\le
\mathcal L(\vect\theta_j)
-\alpha \nabla \mathcal L(\vect\theta_j)^\top \widehat{\nabla}\mathcal L(\vect\theta_j)
+\frac{L_{\mathcal L}\alpha^2}{2}\|\widehat{\nabla}\mathcal L(\vect\theta_j)\|^2.
\]
Taking the conditional expectation and using the unbiased and bounded-variance assumptions yields
\[
\mathbb E[\mathcal L(\vect\theta_{j+1})]
\le
\mathbb E[\mathcal L(\vect\theta_j)]
-\frac{\alpha}{2}\mathbb E[\|\nabla\mathcal L(\vect\theta_j)\|^2]
+\frac{L_{\mathcal L}\alpha^2}{2}\sigma^2,
\]
where the factor $1/2$ follows from $\alpha\le 1/L_{\mathcal L}$. Summing this inequality from $j=0$ to $K-1$ gives the stated bound. The full-batch case follows by setting $\sigma=0$.
\end{proof}

Theorem~\ref{thm:cm_convergence} shows that the proposed cost-matching update converges to a first-order stationary set of \eqref{eq:cm_loss} under Assumption~\ref{assumption}, which is compatible with our setting since training is performed over a bounded rollout buffer.
%%%%%%%%%%%%%%%%%%%%%%%%%%%%%%%%%%%%%%%%%%%%%%%%%%%%%%%%%%%%%%%%%%%%%%%%%%%%%%%%%%%%%%%%%%%%%%%%%
\subsection{Constraints and Feasibility}
\label{sec:cm_constraints}

The treatment of constraints differs between the learning and deployment phases. During training, the MPC-induced action-value function \eqref{eq:cm_Q_mpc_def} incorporates state-related constraints through differentiable penalty terms based on the violation measures defined in \eqref{eq:cm_penalty}. This soft-penalty formulation allows constraint violations to be evaluated along data-driven rollouts without requiring the predicted trajectory to remain feasible, which is essential when the parameterized model does not perfectly capture the true system dynamics.

After training, the learned parameters $\vect\theta^\star$ define a structured MPC controller that is deployed by solving the constrained optimization problem \eqref{eq:cm_mpc} in a receding-horizon manner. In this phase, input and state constraints are enforced as hard constraints, thereby ensuring feasibility and safety in closed-loop operation. Because the learning process shapes the MPC model and cost to match the long-term performance of the real system, the resulting controller benefits from both data-driven adaptation and the constraint-handling guarantees inherent to MPC. This separation—soft constraint handling during learning and hard constraint enforcement during online MPC execution—enables efficient gradient-based training while preserving the safety and feasibility guarantees of MPC at runtime.

%%%%%%%%%%%%%%%%%%%%%%%%%%%%%%%%%%%%%%%%%%%%%%%%%%%%%%%%%%%%%%%%%%%%%%%%%%%%%%%%%%%%%%%%%%%%%%%%%
%%%%%%%%%%%%%%%%%%%%%%%%%%%%%%%%%%%%%%%%%%%%%%%%%%%%%%%%%%%%%%%%%%%%%%%%%%%%%%%%
%%%%%%%%%%%%%%%%%%%%%%%%%%%%%%%%%%%%%%%%%%%%%%%%%%%%%%%%%%%%%%%%%%%%%%%%%%%%%%%%%%%%%%%%%%%%%%%%%
\section{Instantiation on Humanoid Locomotion}
\label{sec:cm_instantiation}

This section instantiates the general cost-matching framework introduced in Section~\ref{sec:cost_matching} for the humanoid locomotion OCP described in Section~\ref{sec:problem_formulation}. Rollouts are collected from real-world execution, and the MPC-induced value is evaluated by replaying the logged MPC commands in the parameterized predictive model initialized at the measured state. We next describe the model and cost parameterizations adopted in this locomotion setting.

%%%%%%%%%%%%%%%%%%%%%%%%%%%%%%%%%%%%%%%%%%%%%%%%%%%%%%%%%%%%%%%%%%%%%%%%%%%%%%%%%%%%%%%%%%%%%%%%%
\subsection{Parameterized Model and Objective}
\label{subsec:cm_instantiation_model}

The predictive model $\vect f_{\vect\theta}$ used in \eqref{eq:cm_rollout_dyn} is obtained by discretizing the centroidal dynamics \eqref{eq:centroidal_dynamics}. To account for systematic mismatch between the predictive model and the real system, we parameterize the centroidal momentum propagation using learnable gains:
\begin{subequations}
\label{eq:cm_theta_centroidal}
\begin{align}
\hspace*{-2mm}\dot{\vect h}_{lin} &=
\vect\theta_{hl}\odot
\sum_{i\in\{L,R\}}\vect f_{c,i} + M\vect g,
\\
\hspace*{-2mm}\dot{\vect h}_{ang} &=
\vect\theta_{ha}\odot
\sum_{i\in\{L,R\}}\Big(\vect r_{c,i}(\vect q)\times \vect f_{c,i} + \vect m_{c,i}\Big),
\end{align}
\end{subequations}
where $\vect\theta_{hl}\in\mathbb R_{>0}^{3}$ and $\vect\theta_{ha}\in\mathbb R_{>0}^{3}$ are included in $\vect\theta$. This parameterization preserves the centroidal dynamics structure while absorbing systematic discrepancies caused by imperfect contact wrench realization, unmodeled compliance, and abstraction effects introduced by the low-level tracking controller.

The parameterized stage cost retains the decomposition in~\eqref{eq:ocp_l} as
$L_{\vect\theta}
=
L_{\mathrm{trac}}(\vect\theta) + L_{\mathrm{base}}(\vect\theta) + L_{\mathrm{com}}(\vect\theta) + L_{\mathrm{swin}}(\vect\theta) + L_{\mathrm{torq}}(\vect\theta),$
and the terminal cost is defined as
$T_{\vect\theta}(\vect x_N)=\|\vect x_N-\vect x_N^{\mathrm{ref}}\|_{\vect Q_f(\vect\theta)}^2$.
Quadratic weights are parameterized through diagonal scalings as
$\vect Q_{\vect\theta}=\mathrm{diag}(\vect\theta_q),\,
\vect R_{\vect\theta}=\mathrm{diag}(\vect\theta_r),\,
\vect {Q_f}_{\vect\theta}=\mathrm{diag}({{\vect\theta}_q}_f), $
where $\vect\theta_q, \vect\theta_r, {{\vect\theta}_q}_f\in\mathbb{R}^{12+n_j}_{\ge 0}$ are learnable vectors. The weight matrices $\vect Q_{\mathrm{base}}(\vect\theta)$, $\vect Q_{\mathrm{com}}(\vect\theta)$, $\vect Q_{\mathrm{sw}}(\vect\theta)$, and $\vect Q_{\mathrm{torq}}(\vect\theta)$ are parameterized analogously.

%%%%%%%%%%%%%%%%%%%%%%%%%%%%%%%%%%%%%%%%%%%%%%%%%%%%%%%%%%%%%%%%%%%%%%%%%%%%%%%%%%%%%%%%%%%%%%%%%
\subsection{Cost-Matching Training Loop}
\label{subsec:cm_instantiation_algo}

Training is performed \emph{on-policy} using the parameterized MPC controller. At iteration $j$, the parameter vector $\vect\theta_j$ defines a closed-loop locomotion policy obtained by solving \eqref{eq:cm_mpc} in a receding-horizon manner. Trajectories are collected by executing the resulting commands through the fixed low-level tracking module within a high-fidelity simulator. From the recorded trajectories, and for each sampled index $k$, we compute the measured return $Q^{\mathrm{meas}}_k$ (data length $M$) and the rollout-based surrogate $Q^{\mathrm{MPC}}_{\vect\theta_j,k}$ (data length $N$). The parameters are then updated by minimizing the squared discrepancy between these quantities, as summarized in Algorithm~\ref{alg:cm_locomotion}.

\begin{algorithm}[htbp]
\caption{On-policy Cost-Matching Training for Humanoid Locomotion.}
\label{alg:cm_locomotion}
\begin{algorithmic}[1]
\small
\STATE \textbf{Input:} initial parameters $\vect\theta_0$, discount $\gamma$, stepsize $\alpha$
\FOR{$j=0,1,2,\dots$}
  \STATE \textbf{Data collection.} Run the closed-loop controller induced by \eqref{eq:cm_mpc} with $\vect\theta_j$ for $M$ steps through the full execution stack, and log
  $\mathcal D_j=\{(\vect s_k,\vect a_k,\ell_k)\}_{k=0}^{M-1}$ with $\vect a_k:=\vect u_k$ and $\ell_k:=L(\vect s_k,\vect a_k)$.
  \STATE Sample a mini-batch of indices $\mathcal B\subset\{0,\dots,M-1\}$.
  \FOR{each $k\in\mathcal B$}
    \STATE \textbf{Measured return.} Compute $Q^{\mathrm{meas}}_k=$\\ $\qquad\qquad\qquad\sum_{i=0}^{M-k-1}\gamma^i\,L(\vect s_{k+i},\vect a_{k+i})$ via \eqref{eq:cm_Q_meas}.
    \STATE \textbf{MPC-side rollout.} Set $\vect x_0:=\vect s_k$ and propagate for $i=0,\dots,N-1$:
    $\vect x_{i+1}=\vect f_{\vect\theta_j}(\vect x_i,\vect a_{k+i})$.
    \STATE \textbf{Surrogate value.} Evaluate $Q^{\mathrm{MPC}}_{\vect\theta_j,k}:=$\\$\qquad\qquad\qquad\qquad Q^{\mathrm{MPC}}_{\vect\theta_j}(\vect s_k,\vect a_{k:k+N-1})$ via \eqref{eq:cm_Q_mpc_def}.
  \ENDFOR
  \STATE Form $\mathcal L(\vect\theta_j)=\frac{1}{|\mathcal B|}\sum_{k\in\mathcal B}\big(Q^{\mathrm{MPC}}_{\vect\theta_j,k}-Q^{\mathrm{meas}}_k\big)^2$.
  \STATE Update $\vect\theta_{j+1}\leftarrow\vect\theta_j-\alpha\nabla_{\vect\theta}\mathcal L(\vect\theta_j)$.
\ENDFOR
\end{algorithmic}
\end{algorithm}

%%%%%%%%%%%%%%%%%%%%%%%%%%%%%%%%%%%%%%%%%%%%%%%%%%%%%%%%%%%%%%%%%%%%%%%%%%%%%%%%%%%%%%%%%%%%%%%%%
\subsection{Implications for Locomotion Control}
\label{subsec:cm_instantiation_benefits}

For bipedal locomotion, this instantiation yields several practical benefits.

\noindent\textbf{(i) Mitigating short-horizon bias.}
By matching $Q^{\mathrm{MPC}}_{\vect\theta}$ to long-horizon returns, the learned stage and terminal shaping terms effectively act as a tail-cost approximation, mitigating the short-horizon bias inherent in real-time humanoid MPC.

\noindent\textbf{(ii) Compensating model mismatch.}
The learned parameters adapt both the prediction model and the objective to systematic discrepancies between the planning model and high-fidelity execution, including modeling inaccuracies and abstraction effects introduced by the tracking controller.

\noindent\textbf{(iii) Robustness to disturbances.}
When trajectories are collected under external disturbances, the measured returns reflect the recovery behaviors of the system. Consequently, the learned MPC objective is implicitly tuned toward policies that improve disturbance rejection under the same constrained MPC structure.

\noindent\textbf{(iv) Policy improvement.}
Although the update minimizes a value discrepancy rather than directly maximizing return, improving the fidelity of the MPC internal cost-to-go makes the receding-horizon minimizer of the surrogate objective more consistent with the true long-term performance, thereby yielding improved closed-loop policies.

\noindent\textbf{(v) Training efficiency.}
Each update differentiates through a length-$N$ rollout without repeatedly solving the MPC optimization problem within the learning loop, resulting in a computational cost that scales linearly with $N$.
%
%%%%%%%%%%%%%%%%%%%%%%%%%%%%%%%%%%%%%%%%%%%%%%%%%%%%%%%%%%%%%%%%%%%%%%%%%%%%%%%%%%%%%%%%%%%%%
\section{Simulation Studies and Analysis}
\label{sec:sim}

This section evaluates the proposed cost-matching MPC (CM-MPC) on a commercial humanoid platform (Unitree G1) using the locomotion stack described in Section~\ref{sec:problem_formulation}. In the following experiments, we compare two controllers operating under identical MPC horizon length, constraints, and tracking stack: (i) a manually tuned baseline using the initial parameters $\vect\theta_0$, and (ii) the learned CM-MPC controller using the converged parameters $\vect\theta^\star$.

%%%%%%%%%%%%%%%%%%%%%%%%%%%%%%%%%%%%%%%%%%%%%%%%%%%%%%%%%%%%%%%%%
\subsection{Cost-Matching Diagnostics}
\label{subsec:diagnostics}

The learning objective is to reduce the discrepancy between the MPC-induced surrogate cost-to-go $Q^{\mathrm{MPC}}_{\vect\theta}$ and the measured long-horizon return $Q^{\mathrm{meas}}$ collected from closed-loop executions. Figure~\ref{fig:cm_mpc_closed_loop_mse} reports the validation mean squared error (MSE) of this value mismatch across training rounds. The mismatch decreases rapidly during the early training iterations and gradually converges as learning progresses.

The block-wise evolution of $\vect\theta$ provides insights into the mechanisms underlying this improvement. In particular, the terminal cost weights ($\vect Q_f$) increase significantly throughout training, effectively acting as an implicit tail-cost approximation that compensates for the inherent myopia of the finite-horizon MPC formulation. Concurrently, torque-regularization parameters increase, which suppresses impulsive control actions, while swing-foot tracking weights decrease slightly. This reshaping of the cost landscape prevents overly strict kinematic tracking and instead allows the solver greater flexibility to prioritize whole-body balance and compliant behavior under uncertainties.

\begin{figure}[htbp]
  \centering
  \includegraphics[width=0.9\linewidth]{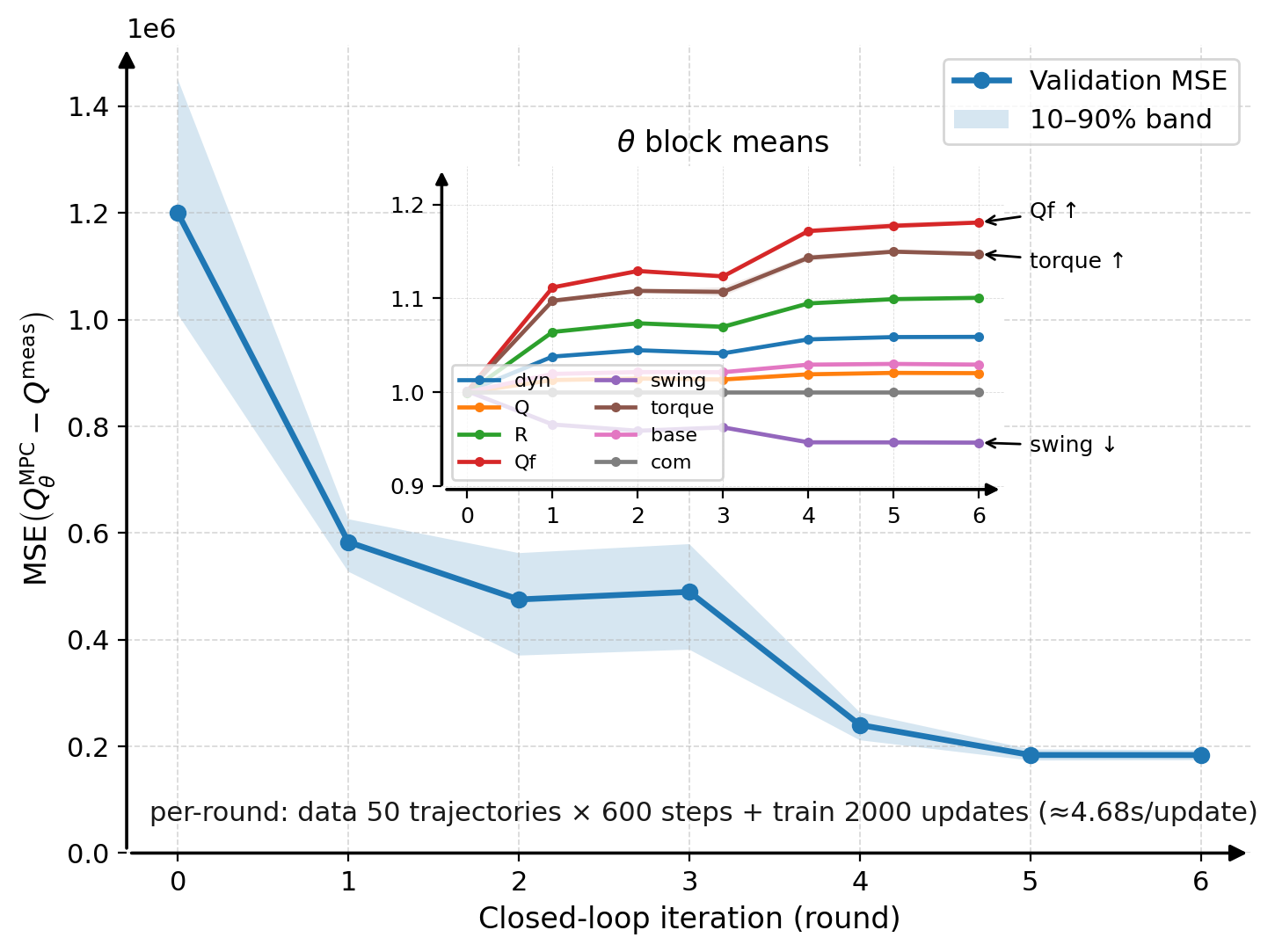}
  \caption{Closed-loop validation of the value mismatch $\mathrm{MSE}\!\left(Q^{\mathrm{MPC}}_{\vect{\theta}}-Q^{\mathrm{meas}}\right)$ across training rounds (shaded region: $10-90$\% across seeds). At each training round, $50$ trajectories are collected ($600$ steps per trajectory), and the cost-matching objective is optimized using $2000$ gradient updates with discount factor $\gamma=0.985$. The inset summarizes the normalized block-wise evolution of $\vect\theta$, illustrating how learning reshapes the MPC objective.}
  \label{fig:cm_mpc_closed_loop_mse}
\end{figure}

Figure~\ref{fig:a2_value_matching} provides complementary diagnostics at the sample level. The figure shows the density of $Q^{\mathrm{MPC}}_{\vect{\theta}}$ versus $Q^{\mathrm{meas}}$ before training and after convergence. After learning, the root-mean-square error (RMSE) between the two values decreases and the residual distribution tightens, indicating improved consistency between the MPC surrogate value and the measured return. Together, these diagnostics confirm that the learned parameters significantly improve the fidelity of the MPC internal cost-to-go approximation.

\begin{figure}[tbhp]
  \centering
  \begin{subfigure}[t]{0.49\linewidth}
    \centering
    \includegraphics[width=\linewidth]{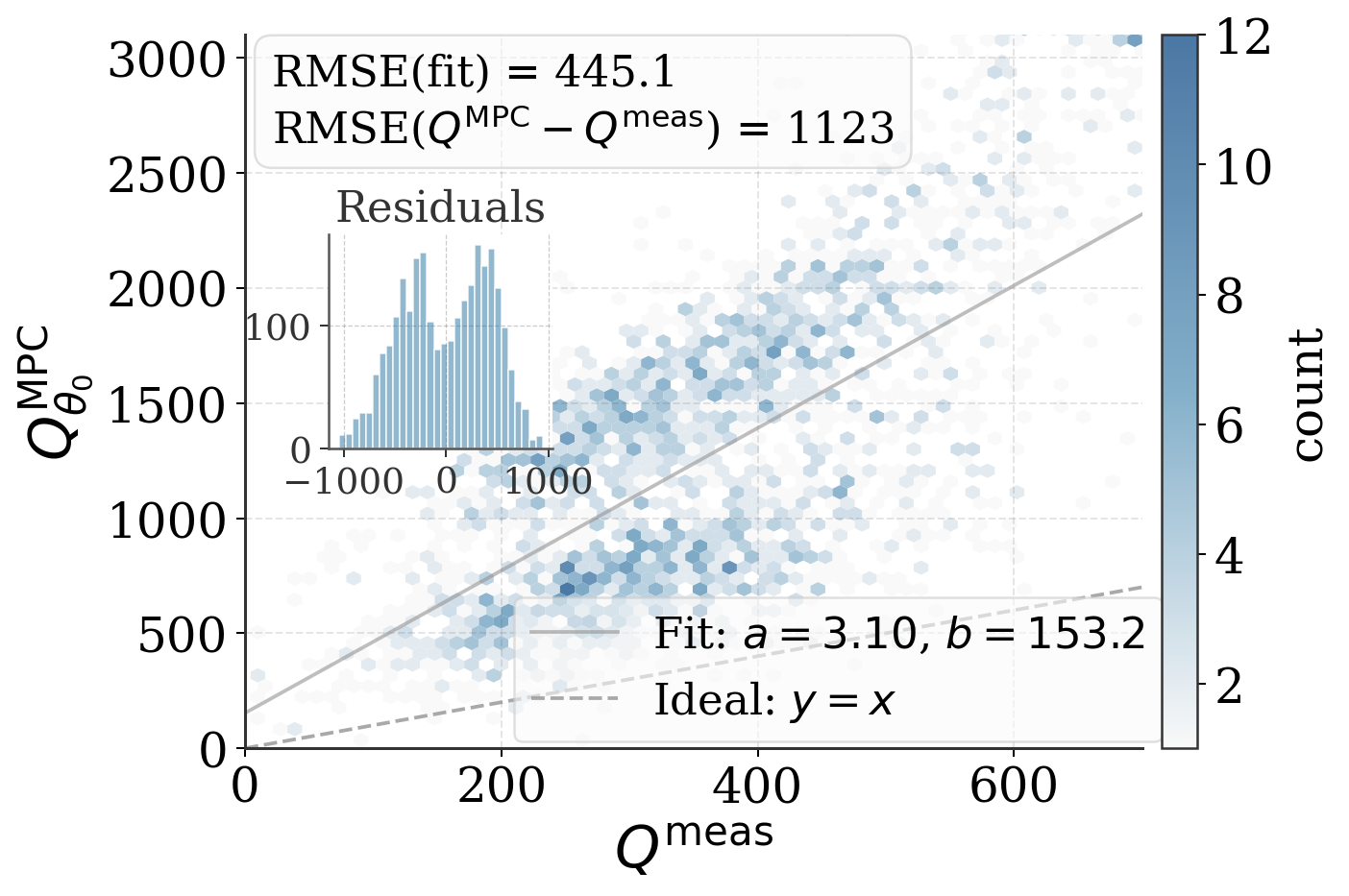}
    \caption{Round 0 ($\vect{\theta}_0$).}
    \label{fig:a2_value_matching_a}
  \end{subfigure}
  \hfill
  \begin{subfigure}[t]{0.49\linewidth}
    \centering
    \includegraphics[width=\linewidth]{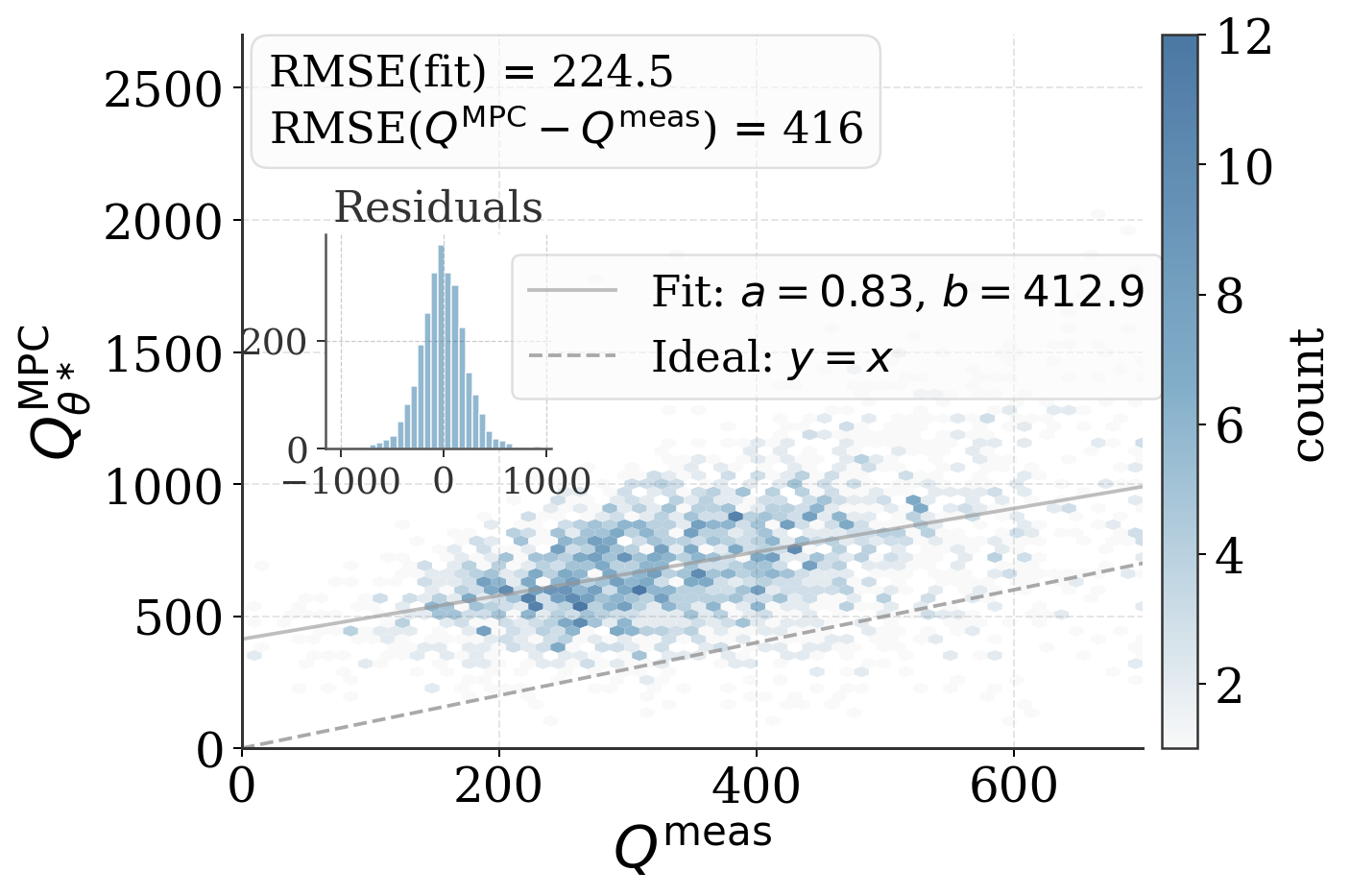}
    \caption{Converged ($\vect{\theta}^{\star}$).}
    \label{fig:a2_value_matching_b}
  \end{subfigure}
  \caption{Value-matching diagnostics on a fixed evaluation set: density plot of $Q^{\mathrm{MPC}}_{\vect{\theta}}$ versus $Q^{\mathrm{meas}}$.}
  \label{fig:a2_value_matching}
  \vspace{-10pt}
\end{figure}
%%%%%%%%%%%%%%%%%%%%%%%%%%%%%%%%%%%%%%%%%%%%%%%%%%%%%%%%%%%%%%%%%
\subsection{Robustness Evaluation under Disturbances}
\label{subsec:robustness}

To assess whether the improved surrogate value representation translates into better closed-loop behavior, we evaluate both the baseline MPC and the learned CM-MPC under a scripted disturbance benchmark applied to the Unitree G1 humanoid. 

As illustrated in Figure~\ref{fig:push_robot}, the disturbance profile consists of three lateral force pulses $f_{\mathrm{ext},y}=18\,\mathrm{N}$ with duration $0.10\,\mathrm{s}$ applied at $t=\{12.00,\,12.75,\,13.55\}\,\mathrm{s}$ (P1/2/4), and one yaw torque pulse $m_{\mathrm{ext},z}=8\,\mathrm{N\,m}$ with duration $0.06\,\mathrm{s}$ at $t=13.05\,\mathrm{s}$ (P3). This mixed force–torque sequence excites both linear and angular momentum channels without modifying the nominal reference trajectories. Consequently, differences in post-disturbance recovery can be attributed to the objective and model reshaping induced by the learned parameters $\vect{\theta}$.

\begin{figure}[htbp]
    \centering
    \includegraphics[width=\columnwidth]{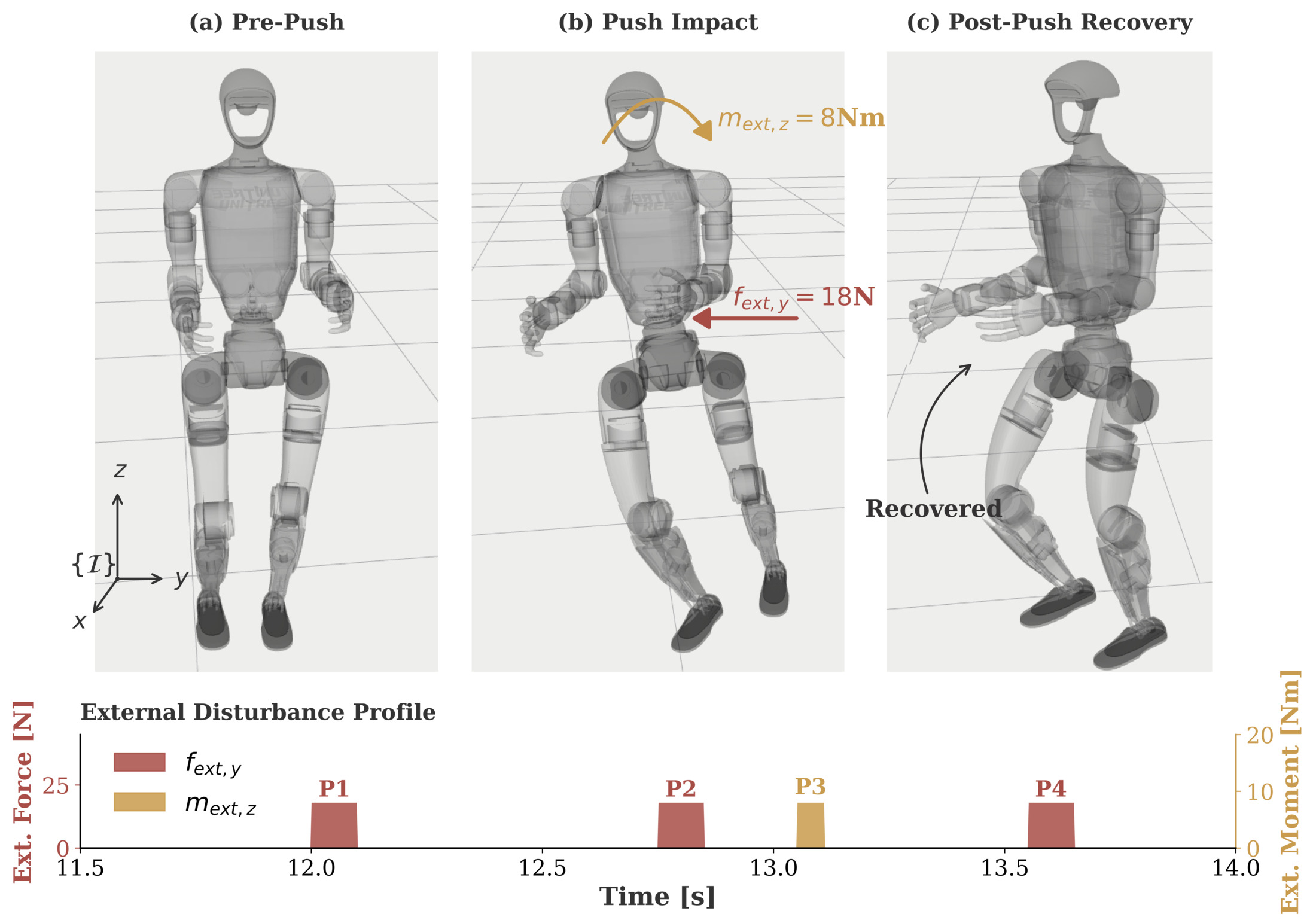}
    \caption{Simulation snapshots of the humanoid during locomotion under a time-scheduled disturbance benchmark consisting of lateral force pulses $f_{\mathrm{ext},y}$ and a yaw-torque pulse $m_{\mathrm{ext},z}$.}
    \label{fig:push_robot}
    \vspace{-5pt}
\end{figure}

Figure~\ref{fig:closed_loop_traj} shows a representative rollout comparing the baseline controller and CM-MPC. From top to bottom, the plots report the tracking error $\|x-x_{\mathrm{ref}}\|$, control effort $\|u\|$, control jitter $\|\dot u\|$ (discrete-time input difference), linear momentum error $\|h_{\mathrm{lin}}-h_{\mathrm{lin,ref}}\|$, and angular momentum error $\|h_{\mathrm{ang}}-h_{\mathrm{ang,ref}}\|$.

Compared with the baseline controller, CM-MPC settles earlier to the pre-disturbance error band and produces smoother corrective actions, as evidenced by reduced control effort $\|u\|$ and lower control jitter $\|\dot u\|$ during the recovery phase. At the same time, the momentum tracking errors do not consistently decrease: $\|h_{\mathrm{lin}}-h_{\mathrm{lin,ref}}\|$ remains broadly comparable, and $\|h_{\mathrm{ang}}-h_{\mathrm{ang,ref}}\|$ may increase during parts of the transient. This behavior reflects a deliberate trade-off induced by the learned parameters $\vect{\theta}^{\star}$. By relaxing strict tracking penalties (e.g., reduced swing-foot tracking weights) while increasing actuation regularization (larger $R$ / torque-related weights), CM-MPC allocates its limited control authority under constraints toward stabilizing whole-body motion and rejecting disturbances, rather than enforcing exact momentum reference tracking at every instant.
\begin{figure*}[htbp]
  \centering
  \includegraphics[width=0.85\linewidth]{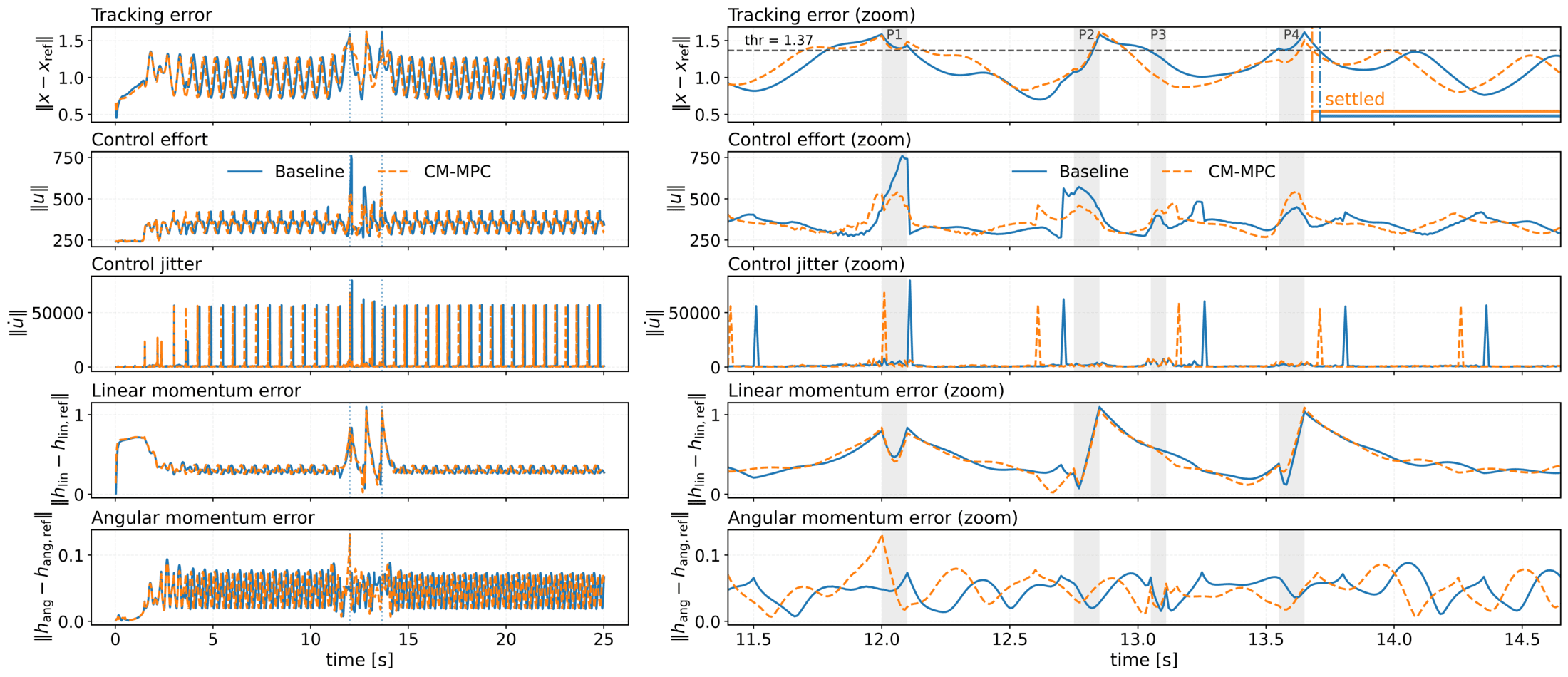}
  \caption{Representative closed-loop response under the disturbance benchmark. Left: full-horizon signals. Right: zoom around the push interval. The gray bands indicate push windows. The dashed horizontal line indicates the settle threshold computed from pre-disturbance statistics, and the ``settled'' marker indicates the first time the error remains below this threshold.}
  \label{fig:closed_loop_traj}
  \vspace{-12pt}
\end{figure*}

%%%%%%%%%%%%%%%%%%%%%%%%%%%%%%%%%%%%%%%%%%%%%%%%%%%%%%%%%%%%%%%%%
\subsection{Comprehensive Statistical Analysis}
\label{subsec:stats}

To validate the observed improvements statistically, Table~\ref{tab:comprehensive_metrics} summarizes results over six disturbance trials (mean $\pm$ std). The learned CM-MPC improves nominal tracking while reducing high-percentile control effort, indicating that the learned parameters $\vect{\theta}$ produce a better trade-off between tracking accuracy and control smoothness under the same constrained MPC structure. More importantly, the post-push recovery metrics show consistent improvements. The \texttt{Post-Push RMS} decreases, while the \texttt{Post-Push Peak} and \texttt{IAE} (Integral Absolute Error) are reduced by $4.80\%$ and $6.82\%$, respectively, indicating smaller overshoot and less accumulated deviation following the disturbance. The \texttt{Settling Time} improves by $31.25\%$ under the same threshold-and-hold criterion, reflecting faster return to the pre-disturbance performance band. Momentum metrics remain largely comparable, with a minor trade-off in linear momentum error ($-0.99\%$). This observation is consistent with the qualitative analysis in Figure~\ref{fig:closed_loop_traj}, where the learned controller prioritizes faster stabilization and smoother actuation during transient recovery rather than enforcing strict reference momentum tracking.

%%%%%%%%%%%%%%%%%%%%%%%%%%
\begin{table}[htbp]
\centering
\caption{Comprehensive Evaluation of Metrics (Mean $\pm$ Std) over Six Trials.}
\label{tab:comprehensive_metrics}
\setlength{\tabcolsep}{3pt}
\resizebox{\columnwidth}{!}{%
\renewcommand{\arraystretch}{1.15}
\begin{tabular}{@{} l c c c @{}}
\toprule
\textbf{Metric} & \textbf{Baseline} & \textbf{CM-MPC} & \textbf{Impr. (\%)} \\
\midrule
\multicolumn{4}{@{}l}{\textit{\textbf{Overall Tracking}}} \\
\cmidrule(r){1-1}
RMS Error ($\|\vect e\|_2$)          & 1.022 $\pm$ 0.012 & 1.009 $\pm$ 0.016 & +1.29 \\
Peak Error ($\|\vect e\|_{\infty}$)  & 1.620 $\pm$ 0.021 & 1.597 $\pm$ 0.045 & +1.41 \\
\addlinespace[0.4em]
\multicolumn{4}{@{}l}{\textit{\textbf{Control Effort \& Smoothness}}} \\
\cmidrule(r){1-1}
Max Effort ($\|\vect u\|_{99\%}$)            & 472.7 $\pm$ 4.3   & 447.8 $\pm$ 7.1   & +\textbf{5.26} \\
Max Jitter ($\|\dot{\vect u}\|_{99\%}$) [$10^4$] & 5.66 $\pm$ 0.00 & 5.60 $\pm$ 0.03 & +1.10 \\
\addlinespace[0.4em]
\multicolumn{4}{@{}l}{\textit{\textbf{Momentum Analysis}}} \\
\cmidrule(r){1-1}
Lin. Mom. ($\|\vect h_{\mathrm{lin}}\|_2$)      & 0.468 $\pm$ 0.011 & 0.467 $\pm$ 0.011 & +0.35 \\
Ang. Mom. ($\|\vect h_{\mathrm{ang}}\|_2$)     & 0.048 $\pm$ 0.001 & 0.048 $\pm$ 0.001 & +1.18 \\
Lin. Mom. Err. ($\|\vect h_{\mathrm{lin}} - \vect h_{\mathrm{lin}}^{\mathrm{ref}}\|_2$) & 0.372 $\pm$ 0.010 & 0.376 $\pm$ 0.010 & -0.99 \\
Ang. Mom. Err. ($\|\vect h_{\mathrm{ang}} - \vect h_{\mathrm{ang}}^{\mathrm{ref}}\|_2$) & 0.048 $\pm$ 0.001 & 0.048 $\pm$ 0.001 & +1.18 \\
\addlinespace[0.4em]
\multicolumn{4}{@{}l}{\textit{\textbf{Post-Push Recovery}}} \\
\cmidrule(r){1-1}
Post-Push RMS ($\|\vect e^{\mathrm{post}}\|_2$)      & 1.026 $\pm$ 0.010 & 1.014 $\pm$ 0.015 & +1.17 \\
Post-Push Peak ($\|\vect e^{\mathrm{post}}\|_{\infty}$)& 1.539 $\pm$ 0.030 & 1.465 $\pm$ 0.020 & +\textbf{4.80} \\
Integral Abs. Err. (IAE)               & 0.547 $\pm$ 0.033 & 0.510 $\pm$ 0.030 & +\textbf{6.82} \\
Settling Time ($t_s$)                  & 0.053 $\pm$ 0.010 & 0.037 $\pm$ 0.010 & +\textbf{31.25} \\
\bottomrule
\end{tabular}%
}
\vspace{0.3em} 
\caption*{\footnotesize
\textit{$\divideontimes$ Post-push recovery metrics.} Let $\vect e^{\mathrm{post}}(t)\triangleq\|\vect x(t)-\vect x_{\mathrm{ref}}(t)\|$ evaluated over the post-push window $t\in[t_{\mathrm{push\_end}},\,t_{\mathrm{push\_end}}+T_{\mathrm{post}}]$. Define $\texttt{Post-Push RMS}=\|\vect e^{\mathrm{post}}\|_2$, $\texttt{Post-Push Peak}=\|\vect e^{\mathrm{post}}\|_{\infty}$, and
$\texttt{IAE}=\int \max(\vect e^{\mathrm{post}}-\bar{\vect e}_{\mathrm{pre}},0)\,dt$ (where $\bar{\vect e}_{\mathrm{pre}}$ is the pre-push baseline). \texttt{Settling Time} $t_s$ is the first time after $t_{\mathrm{push\_end}}$ that $\vect e^{\mathrm{post}}(t)$ remains within a pre-push threshold for a fixed hold duration.}
\vspace{-10pt}
\end{table}

%%%%%%%%%%%%%%%%%%%%%%%%%%%%%%%%%%%%%%%%%%%%%%%%%%%%%%%%%%%%%%%%%%%%%%%%%%%%%%%%

Overall, the simulation results demonstrate that the cost-matching procedure effectively minimizes the discrepancy between $Q^{\mathrm{MPC}}_{\vect\theta}$ and $Q^{\mathrm{meas}}$ throughout training. The optimized parameters $\vect\theta^\star$ lead to improved disturbance rejection characterized by faster recovery, reduced control effort, and smoother actuation, while simultaneously compensating for modeling discrepancies in the simplified centroidal dynamics through data-driven adaptation.
%%%%%%%%%%%%%%%%%%%%%%%%%%%%%%%%%%%%%%%%%%%%%%%%%%%%%%%%%%%%%%%%%%%%%%%%%%%%%%%%%%%%%%%%%
\section{Conclusion}
\label{sec:conclusion}
This paper presents a cost-matching framework for learning the parameters of a constrained centroidal MPC controller for humanoid locomotion. By matching the MPC surrogate cost-to-go to measured long-horizon returns from closed-loop rollouts, the method avoids repeated MPC solves in the learning loop. Simulations show improved value consistency and enhanced robustness to model mismatch and external disturbances, suggesting that cost matching is a practical and computationally efficient way to shape long-horizon locomotion behavior while retaining standard constrained MPC deployment.

%%%%%%%%%%%%%%%%%%%%%%%%
\bibliographystyle{IEEEtran}
\bibliography{References}
\end{document}